\documentclass[letterpaper, 10pt, conference]{ieeeconf}

\IEEEoverridecommandlockouts                          
\makeatletter

\let\proof\@undefined
\let\endproof\@undefined
\makeatother

\usepackage[switch]{lineno}

\usepackage[small]{caption}
\usepackage{subcaption}
\usepackage{footmisc}
\usepackage{url}
\usepackage{cite}

\usepackage{pgfplots}
\pgfplotsset{compat=1.7}

\usepackage{graphicx}
\usepackage{booktabs}
\usepackage{makecell}

\usepackage{amsmath}
\usepackage{amsfonts}
\let\labelindent\relax
\usepackage{enumitem}

\usepackage{mathtools, stackengine}

\usepackage{amsthm}
\newtheorem{definition*}{Definition}
\newtheorem{assumption*}{Assumption}
\newtheorem{theorem*}{Theorem}
\newtheorem{lemma}{Lemma}
\newtheorem{problem}{Problem}
\newtheorem{proposition}{Proposition}

\newcommand{\blue}{
}

\begin{document}
% \title{Minimum Turn Coverage of Arbitrary Non-Convex Indoor Environments}
\title{Optimal Partitioning of Non-Convex Environments \\ for Minimum Turn Coverage Planning}

\urldef{\ramesh}\url{m5ramesh@uwaterloo.ca}
\urldef{\smith}\url{stephen.smith@uwaterloo.ca}
\urldef{\fidan}\url{fidan@uwaterloo.ca}
\urldef{\imeson}\url{frank.imeson@avidbots.com}

\author{Megnath Ramesh, Frank Imeson, Baris Fidan, and Stephen L. Smith
\thanks{M. Ramesh (\ramesh) and S. L. Smith (\smith) are with the Department of Electrical and Computer Engineering and B. Fidan (\fidan) is with the Department of Mechanical and Mechatronics Engineering, at the University of Waterloo, Waterloo ON, Canada}
\thanks{F. Imeson (\imeson) is with Avidbots Corp., Kitchener ON, Canada}
}

\maketitle

\begin{abstract}
In this paper, we tackle the problem of planning an optimal coverage path for a robot operating indoors. Many existing approaches attempt to discourage turns in the path by covering the environment along the least number of \emph{coverage lines}, i.e., straight-line paths. This is because turning not only slows down the robot but also negatively affects the quality of coverage, e.g., tools like cameras and cleaning attachments commonly have poor performance around turns. The problem of minimizing coverage lines however is typically solved using heuristics that do not guarantee optimality. In this work, we propose a turn-minimizing coverage planning method that computes the optimal number of axis-parallel (horizontal/vertical) coverage lines for the environment in polynomial time. We do this by formulating a linear program (LP) that optimally partitions the environment into axis-parallel \emph{ranks} (non-intersecting rectangles of width equal to the tool width). We then generate coverage paths for a set of real-world indoor environments and compare the results with state-of-the-art coverage approaches.
\end{abstract}
% \IEEEpeerreviewmaketitle
\section{Introduction}
Coverage path planning is an automation challenge in which a robot must find an optimal path such that its tool or sensor covers the entire environment \cite{galceranSurveyCoveragePath2013}. This problem has a wide range of applications, including cleaning \cite{hofnerPathPlanningGuidance1995}, agriculture \cite{hameedIntelligentCoveragePath2014}, visual inspection \cite{songOnlineCoverageInspection2020, jingCoveragePathPlanning2019, biundiniFrameworkCoveragePath2021}, and more recently, autonomous disinfection of hospitals during the COVID-19 pandemic \cite{nasirianEfficientCoveragePath2021}. The primary objective in coverage planning is to minimize overlap or ``double'' coverage. This is typically achieved through a lawnmower-style path consisting of parallel non-overlapping coverage lines. Thus, the main differentiation between coverage plans is the transitions between these lines, where the robot performs little to no additional coverage. This has resulted in a growing body of work focused on \emph{minimizing turns}, i.e., minimizing the time spent transitioning between coverage lines. Turns also have the following adverse effects: (i) the robot travels slower around turns which increases total coverage time, (ii) in cleaning applications, the tool does not properly pick up water and dust while turning, and (iii) in sample retrieval applications, the robots may experience high pose estimation errors and poor sensor coverage quality during turns \cite{dasMappingPlanningSample2014}.

% In this paper, like in \cite{bochkarevMinimizingTurnsRobot2016} and \cite{vandermeulenTurnminimizingMultirobotCoverage2019}, we prioritize minimizing the number of turns taken by the robot.}

% The quality of the plan can be measured in terms of a variety of factors; notably the area covered, the time taken to cover it and the number of turns taken by the robot. 

\begin{figure}[t!]
\vspace{0.055in}
\centering
\subcaptionbox{}
{\centering
\includegraphics[width=0.44\linewidth]{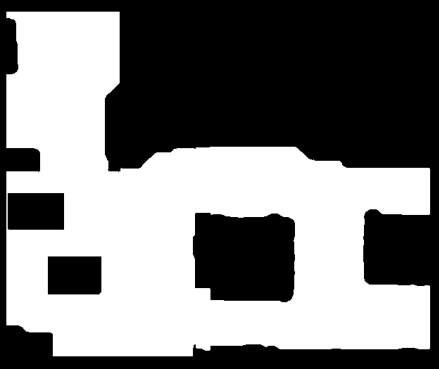}}
\hfill
\subcaptionbox{}
{\centering
\includegraphics[width=0.44\linewidth]{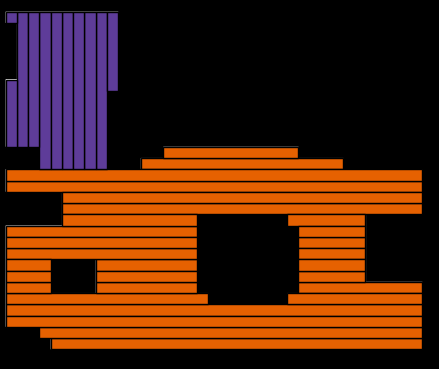}}
\hfill
% \subcaptionbox{Coverage Plan}
% {\includegraphics[width=0.3\linewidth]{Assets/map_1_paths_coloured.png}}
% \caption{Comparison of ranks generated in a region.}
% % % Make the lines coloured
\caption{(a) An example environment and (b) its optimal rank partitioning obtained using the proposed OARP method: the orange partitions are horizontal ranks and the purple are vertical ranks.
\vspace{-0.5cm}}
\label{fig:rank-part}
\end{figure}

% \begin{figure}[t!]
% \vspace{0.055in}
% \centering
% \includegraphics[width=0.95\linewidth]{Assets/main-fig.pdf}
% % % Make the lines coloured
% \caption{An example coverage plan for a cleaning robot. The effect of a turn is also illustrated, where the cleaning tool performs poorly while turning (wavy region) as opposed to straight-line coverage (grey region).}
% \label{fig:cleaning_robot}
% \vspace{-12pt}
% \end{figure}

% An ideal coverage plan covers all feasible regions with the tool in the shortest time. 
Coverage planning with minimum turns is however an NP-Hard problem \cite{arkinOptimalCoveringTours2005}. As such, a common solution framework is to simplify the problem into three steps:
\begin{enumerate}[label=(\roman*)]
    \item \emph{decompose} the environment into sub-regions or \emph{cells},
    \item use the cells to help \emph{place coverage lines} (straight-line paths) in the environment, and
    \item construct the coverage path by computing a \emph{tour} of the coverage lines.
\end{enumerate}
For example, exact decomposition methods like \cite{chosetCoveragePathPlanning1998} decompose the environments into convex cells, each of which is covered by a lawnmower-style path. Following this framework, the robot only turns when transitioning from one coverage line to another.{\blue As a result, the number of turns in the path is often reduced by minimizing the number of coverage lines.}

A state-of-the-art method which follows this general framework is Vandermeulen et. al. \cite{vandermeulenTurnminimizingMultirobotCoverage2019}, where they aim to cover the environment along a minimum number of axis-parallel (horizontal/vertical) coverage lines. Such a method is particularly well suited for indoor environments, in which many walls are either parallel or orthogonal. Their method decomposes the environment into cells, which are then used to partition the environment into \textit{ranks}, i.e., thin disjoint rectangles of width equal to the robot’s tool width, where the center line of each rank along its length defines a coverage line. However, to solve the partitioning step, the authors propose an iterative heuristic with no optimality guarantees. In this paper, we propose a coverage planner that follows the above framework and plans axis-parallel coverage paths with theoretical guarantees and with better performance.

\textbf{\emph{Contributions:}} Our specific contributions are as follows:
\begin{enumerate}
    \item We prove that the axis-parallel rank partitioning problem is tractable and propose a polynomial-time solver that is guaranteed to find the optimal partition. We do this by posing the problem as a mixed integer linear program (MILP) and then proving that the linear relaxation provides optimal solutions.
    \item We develop the \textbf{O}ptimal \textbf{A}xis-Parallel \textbf{R}ank \textbf{P}artitioning (OARP) method that leverages the above solver and plans a coverage path for non-convex environments.
    \item {\blue We present experimental results using maps of real-world environments and show that OARP outperforms the state-of-the-art method from \cite{vandermeulenTurnminimizingMultirobotCoverage2019} in minimizing the decomposition time, number of ranks, number of turns, and consequently the cost of the coverage plan.} 
\end{enumerate}

%{\blue For this work, our main assumption is that we consider a square coverage tool of width $l$, similar to \cite{arkinOptimalCoveringTours2005} and \cite{vandermeulenTurnminimizingMultirobotCoverage2019}.}
% Our method utilizes a simpler grid-based decomposition of the environment, which is faster to compute than \cite{vandermeulenTurnminimizingMultirobotCoverage2019}'s decomposition step.

\textbf{\emph{Related Work:}} Apart from \cite{vandermeulenTurnminimizingMultirobotCoverage2019}, there are coverage planning methods in the literature that follow the same framework.{\blue Detailed surveys of such coverage planning methods can be found in \cite{galceranSurveyCoveragePath2013} and \cite{bormannIndoorCoveragePath2018}.} One category of approaches is \emph{exact decomposition}, where the environment is directly decomposed into (usually convex) cells without any prior approximation. Boustrophedon \cite{chosetCoveragePathPlanning1998} is a widely used method of this category, where each cell is individually covered in a back-and-forth sweeping pattern along parallel coverage lines. There are also turn-minimizing variations of boustrophedon which aim to minimize the number of turns to cover each cell by determining a locally optimal coverage orientation for the cell. These methods use different criteria to determine a cell's orientation: by analysing the geometry of each cell \cite{bochkarevMinimizingTurnsRobot2016, bahnemannRevisitingBoustrophedonCoverage2021, dasMappingPlanningSample2014} or its convex hull \cite{torresCoveragePathPlanning2016}, or minimizing a cost function that models the robot's turns \cite{j.jinOptimalCoveragePath2010}. While these variations work well to incorporate multiple coverage directions to minimize turns, cell decomposition can be time-intensive for complex real-world environments. Also, these approaches do not guarantee that the total number of turns in the path is minimized. In this paper, we address these shortcomings by creating an optimal axis-parallel rank partition of the environment that implicitly minimizes the total number of turns without incorporating complex cell decomposition.

Another category of approaches, referred to as \emph{grid-based approaches}, involve decomposing the environment into a set of uniform \emph{grid cells}, where the robot must cover all accessible grid cells.{\blue Grid decomposition approximates the environment's boundary and obstacles and captures the areas coverable by the tool.} Our proposed approach falls under this category. The shape of the grid cells is usually a square \cite{cabreiraGridBasedCoveragePath2019} but there are approaches that use hexagons \cite{kanOnlineExplorationCoverage2020}. Coverage planning for a grid environment can be solved by generating a spanning-tree of the grid cells \cite{gabrielySpanningtreeBasedCoverage2001} to find minimum length paths.{\blue A tour of the grid cells can also be obtained by solving the Travelling Salesman Problem (TSP) to generate a coverage plan \cite{bormannIndoorCoveragePath2018}.} Recently, learning-based coverage planning methods have also been introduced to solve this problem \cite{theileUAVCoveragePath2020a, kyawCoveragePathPlanning2020, apuroopReinforcementLearningBasedComplete2021a,  krishnalakshmananCompleteCoveragePath2020}. These approaches however do not minimize the number of turns while our proposed method does.

% heuristic-based methods are not guaranteed to reach the optimal solution for the problem, whereas our method solves this problem optimally.

\textbf{\emph{Organization:}} In Section \ref{section:problem_def}, we define our decomposition approach and introduce the rank partitioning problem. In Section \ref{section:milp}, we formulate this problem as a Mixed Integer Linear Program (MILP) and then prove in Section \ref{section:poly} that this problem can be relaxed to an LP and optimally solved in polynomial time. In Section \ref{section:gtsp}, we briefly discuss how we compute a tour of the ranks using a Generalized Travelling Salesman Problem (GTSP) solver to obtain the full coverage plan. In Section \ref{section:results}, we present our experimental results on a set of real-world environments and compare the performance of our algorithm with the method proposed in \cite{vandermeulenTurnminimizingMultirobotCoverage2019}.

\section{Problem Definition} \label{section:problem_def}
In this section, we provide a brief background on coverage planning and introduce a simple decomposition of the environment to tackle turn-minimization. Finally, we define the partitioning and touring problems which we will be solving in the rest of the paper.

\subsection{Coverage Planning Problem}

% what is the goal of coverage planning?

Consider a closed and bounded set $\mathcal{W} \subseteq \mathbb{R}^2$ representing all points within the boundaries of our 2D indoor environment. Let $\mathcal{O} \subset \mathcal{W}$ denote the set of obstacles or inaccessible points in $\mathcal{W}$. The set  $\widetilde{\mathcal{W}} = \mathcal{W} \setminus \mathcal{O}$ therefore contains all accessible points in the environment. The goal of the coverage problem is to plan a path so that the robot's tool covers all accessible points in the given environment. The feasibility of this problem however depends on the tool's footprint, which we represent as $\mathcal{A}(\omega) \subset \mathbb{R}^2$ relative to the robot's position $\omega \in \mathbb{R}^2$. Conventionally, $\mathcal{A}(\omega)$ is represented by a simple geometric shape such as a square of a fixed width $l > 0$ centered at $\omega$ \cite{galceranSurveyCoveragePath2013, arkinOptimalCoveringTours2005, gabrielySpanningtreeBasedCoverage2001, vandermeulenTurnminimizingMultirobotCoverage2019}.
{\blue\begin{assumption*}
\label{square-tool-assumption}
The coverage tool is a square with width $l$.
\end{assumption*}}
The square tool may not be able to cover the environment entirely as there will likely be coverage gaps near the boundaries. Let $\widehat{\mathcal{W}} \subseteq \widetilde{\mathcal{W}}$ be the points of the environment that can be covered by the tool. The goal of the coverage problem is then to plan a path $P \subseteq \widetilde{\mathcal{W}}$ such that
\[\bigcup\limits_{\omega_i \in P} \mathcal{A}(\omega_i) = \widehat{\mathcal{W}} \textrm{.}\]

% For this work, we adopt a square coverage tool of width $l$.

\subsection{IOP Decomposition}

\begin{figure}%[t!]
\centering
\includegraphics[width=0.22\textwidth]{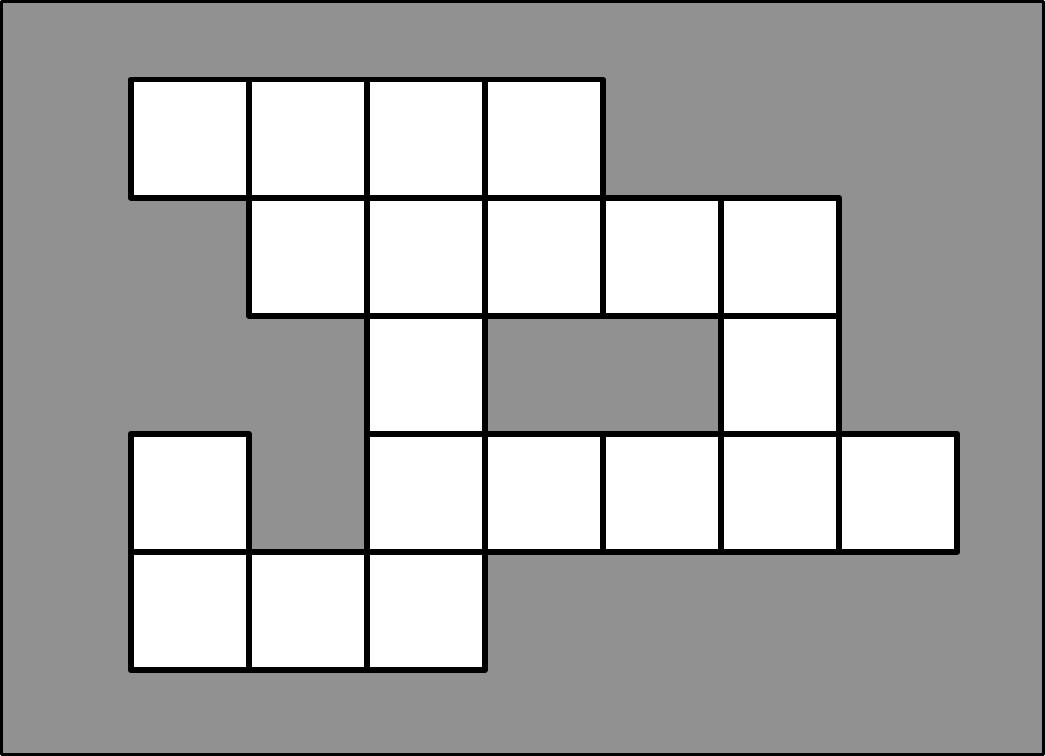}
\caption{An Integral Orthogonal Polygon with light squares representing the polygon and grey areas representing the boundaries and holes.{\blue Each grid cell is as wide as the coverage tool.}
\vspace{-0.5cm}}
\label{fig:ILPGraph}
\end{figure}

% Since this extended problem is intractable, let us now consider a set of orthogonal axes in the plane of the environment and decompose the environment $\widehat{\mathcal{W}}$ using a square grid

In this paper, we look at an extension of this problem where we minimize the number of turns in $P$.{\blue Following from Assumption \ref{square-tool-assumption}, we decompose the environment $\widehat{\mathcal{W}}$ using a square grid and focus on paths where the tool moves only in axis-parallel directions during coverage (horizontal and vertical).} We refer to this decomposition as an Integral Orthogonal Polygon (IOP), which is a set of square \emph{grid cells} of size $l \times l$ \cite{arkinOptimalCoveringTours2005}. Fig.~\ref{fig:ILPGraph} shows an example IOP generated for a simple non-convex environment with one hole. Computing the IOP representation of the environment constitutes the decomposition step of the OARP method.

\subsection{Minimum Rank Partitioning Problem}
 
% The grid cells are as wide as the coverage tool, where $l$ is also the size of the coverage tool. 
We now consider a partition that will help quantify the number of potential turns in the final coverage plan. We look at partitioning the environment into regions that will each be covered with a straight-line path (no turns). A \emph{rank} is the cumulative footprint of the coverage tool when traversing a straight-line path. For a square coverage tool of width $l$, a rank is a rectangular region of width $l$. 

{\blue In an IOP, the robot can perform coverage by moving its square tool from one grid cell to another. A rank in an IOP is therefore a series of horizontally or vertically adjacent grid cells (e.g. Fig. \ref{fig:rankintuition}). Since each grid cell is of the same size as the coverage tool, a rank can be covered along a single coverage line (straight-line path) with endpoints at the centers of grid cells. We can thus define the main problem of this paper as follows.}

%{\blueIn an IOP, we consider robot motion where the square tool starts at the center of a grid cell and moves from one grid cell to another while performing coverage. Thus, all coverage lines of the robot have endpoints at the center of the IOP grid cells and each rank in the IOP consists of a series of horizontally or vertically connected grid cells. }

%  and traverses through the IOP by moving from one grid cell to another

%  The coverage tool starts at the center of a grid cell and traverses through the IOP by moving from one grid cell to another. Similar to the assumption in \cite{arkinOptimalCoveringTours2005}, all paths of the robot have endpoints at the center of the IOP grid cells.

\begin{problem}[Minimum Rank Partitioning Problem]
Given an IOP representation of $\widehat{\mathcal{W}}$, compute a partition of the IOP into the minimum number of axis-parallel (horizontal and vertical) ranks.
\label{problem:main}
\end{problem}

From the rank partition determined in Problem \ref{problem:main}, we generate a coverage path for the environment by solving the following problem.

\begin{problem}[Tour Generation Problem]
Given a set of axis-parallel ranks (partitioning of an IOP), compute a tour of the ranks that minimizes the total cost of transitions between ranks.
\label{problem:tour}
\end{problem}

{\blue In a tour of the ranks, the coverage tool moves from one rank's endpoint to another along transition paths determined by the robot's dynamics.}

% In a tour of the ranks, each rank is covered by a coverage line that have endpoints at the center of the IOP grid cells.

%{\blue We now make some design choices to ensure our IOP can be partitioned into non-intersecting axis-parallel ranks. We assume that the square tool starts at the center of a grid cell and moves from one grid cell to another. Similar to \cite{arkinOptimalCoveringTours2005}, we also assume that all coverage lines of the robot have endpoints at the center of the IOP grid cells. By these choices, each rank in the IOP will contain a series of horizontally or vertically connected grid cells. We can thus define the main problem of this paper as follows.}

\section{MILP Formulation} \label{section:milp}

In this section, we formulate Problem \ref{problem:main} as a mixed integer linear program (MILP), which involves a set of variables and linear constraints. Consider an IOP with $n$ grid cells, each of which is represented by a variable $c_i$ in $C = \{c_1,c_2,\ldots,c_n\}$. We start by introducing an operator $\mathrm{orient}(c_i) \in \{H, V\}$ for each grid cell $c_i$ to represent its coverage orientation ($H$ for horizontal or $V$ for vertical). We also introduce a series of variables that bound the number of ranks in the final partition. The goal of this MILP is to compute the orientations of grid cells that minimize the number of ranks.

Next, we use the concept of neighbouring grid cells to help represent each rank and construct the objective function to minimize the number of ranks. We do this by formulating constraints that can be interpreted as a series of merges for neighbouring cells. We call a grid cell \emph{merge-able} with its neighbour if: (i) the cell and its horizontal neighbour are horizontally oriented, or (ii) the cell and its vertical neighbour are vertically oriented.{\blue A rank is obtained by merging a set of merge-able cells, and so the rank partition of the IOP depends on the assigned cell orientations.} Fig.~\ref{fig:rankintuition} illustrates the intuition of determining ranks from orientations.

\begin{figure}%[t!]
\centering
\includegraphics[width=0.4\textwidth]{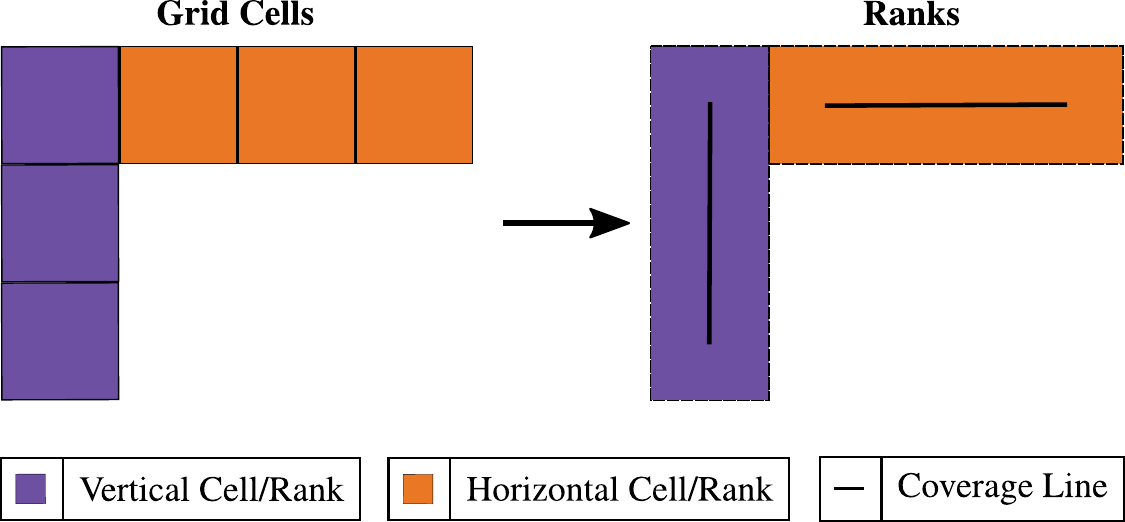}
\caption{Illustration of merging oriented grid cells to form ranks and the corresponding coverage lines.{\blue The cell orientations are determined by the MILP from Section \ref{section:milp}, and the merge-able cells are merged.}
\vspace{-0.5cm}}
\label{fig:rankintuition}
\end{figure}

We now work towards constructing a set of operators that count the ranks for a given set of oriented cells. The following lemma highlights a useful relationship for this task.
\begin{lemma}
In any minimal rank partition, each rank terminates at cells with less than two merge-able neighbours.
\label{lemma:mergability}
\vspace{-0.4cm}
\end{lemma}
\begin{proof}
By definition, each grid cell in the IOP can have at most two merge-able neighbours across an axis. Let $R = (r_1, r_2, \ldots, r_k)$ be an ordered set of $k$ similarly-oriented cells $r_i \in C$ representing a rank along an axis. By our definition of $R$, these cells are connected in a sequential order $(r_1, r_2, \ldots, r_k)$.

Consider one of the endpoints of $R$: $r_1$. By definition, $r_1$ and $r_2$ are merge-able neighbours. Let $r_1$ have another neighbour $r_t$ that is covered by a rank $\tilde{R}$. If $r_t$ is a merge-able neighbour of $r_1$, then $R$ and $\tilde{R}$ can be merged to a single rank, thereby reducing the number of ranks and contradicting the optimality of the partition. This would mean one of two things: (i) $r_t$ does not exist, i.e., $r_t$ is a border or an obstacle, or (ii) $r_t$ is not a merge-able neighbour. Similarly, we can show that the other endpoint $r_k$ has only one merge-able neighbour $r_{k-1}$, which completes the proof of Lemma \ref{lemma:mergability}. 
\end{proof}

% To calculate the exact number of ranks, 
Given the above lemma, we can calculate the number of ranks in an optimal partition by counting one of the two unique endpoints of each rank. For the rest of this paper, we only consider detecting and counting the \emph{left} and \emph{top} endpoints of horizontal and vertical ranks respectively.{\blue The choice of left and top endpoints is arbitrary and does not effect the solution approach.} To identify and count the rank endpoints at the IOP border or holes, we add some artificial cells next to the border cells. These cells are not assigned an orientation, but simply indicate whether there is a rank that has an endpoint at the border. We will refer to these cells as \emph{border identifiers}.

The following logical statements summarize how we will constrain our endpoint counting variables. For each grid cell $c_i$, denoting the left and top neighbors of $c_i$ by $\mathrm{left}(c_i)$ and $\mathrm{top}(c_i)$ respectively, we define our endpoint operators as follows:
\[\mathrm{end_H}(c_i) = \begin{cases*} 
      \hspace{1mm} 1 & \parbox[t]{.25\textwidth}{\raggedright if $\mathrm{orient}(c_i) = H$ and $\mathrm{orient}(\mathrm{left}(c_i)) = V$ }\\[0.1ex]
      \hspace{1mm} 1 & \parbox[t]{.25\textwidth}{\raggedright if $\mathrm{orient}(c_i) = H$ and $\mathrm{left}(c_i)$ is a border identifier} \\[0.1ex]
      \hspace{1mm} 0 & otherwise
   \end{cases*} \textrm{,} \]
\vspace{1pt}
\[\mathrm{end_V}(c_i) = \begin{cases*} 
      \hspace{1mm} 1 & \parbox[t]{.25\textwidth}{\raggedright if $\mathrm{orient}(c_i) = V$ and $\mathrm{orient}(\mathrm{top}(c_i)) = H$} \\[0.1ex]
      \hspace{1mm} 1 & \parbox[t]{.25\textwidth}{\raggedright if $\mathrm{orient}(c_i) = V$ and $\mathrm{top}(c_i)$ is a border identifier}\\[0.1ex]
      \hspace{1mm} 0 & otherwise
   \end{cases*} \textrm{,}\]
where $\mathrm{end_H}(c_i)$ and $\mathrm{end_V}(c_i)$ are binary operators determining if $c_i$ is a horizontal or a vertical endpoint respectively. We now use these functions to formulate the following binary programming problem:
\begin{align}
\min \quad & \sum_{i=0}^{n}\mathrm{end_H}(c_i) + \sum_{i=0}^{n}\mathrm{end_V}(c_i) \label{eq:objective}\\ 
\textrm{s.t.} \quad & \mathrm{orient}(c_i) \in \{H, V\}, \forall i \in \{1, 2, .., n\} \label{eq:constraint} \textrm{.}
\end{align}

Next, we will convert the above binary programming problem to a MILP by replacing the operators with a set of new variables. Let us start with $\mathrm{end_H}(c_i)$ and create the auxiliary variable $y_h^i$ that bounds $\mathrm{end_H}(c_i)$. The upper index ($i$) of $y_h^i$ matches the lower index of $c_i$. We also introduce a set of binary variables to denote whether a cell is horizontally or vertically oriented. For the horizontal case, we introduce $x_h^i$ which equals 1 if $\mathrm{orient}(c_i) = H$ and 0 otherwise. For the left neighbour, $c_l = \mathrm{left}(c_i)$, we observe the following:
\[ x_h^i - x_h^l  = \begin{cases*} 
      \hspace{1mm} 1 & \parbox[t]{.27\textwidth}{\raggedright if $\mathrm{orient}(c_i) = H$ and $\mathrm{orient}(\mathrm{left}(c_i)) = V$} \\[0.1ex]
      \hspace{1mm} 1 & \parbox[t]{.27\textwidth}{\raggedright if $\mathrm{orient}(c_i) = H$ and $\mathrm{left}(c_i)$ is a border identifier} \\[0.1ex]
      \hspace{1mm} -1 & \parbox[t]{.27\textwidth}{\raggedright if $\mathrm{orient}(c_i) = V$ and $\mathrm{orient}(\mathrm{left}(c_i)) = H$} \\[0.1ex]
      \hspace{1mm} 0 & otherwise
   \end{cases*} \textrm{.}\]
The above, along with a non-negativity constraint, gives the following encoding of $y_h^i$:
\begin{align*}
y_h^i &\geq x_h^i - x_h^l\\
y_h^i &\geq 0\
\textrm{,}
\end{align*}
where $y_h^i \geq \mathrm{end_H}(c_i)$. With these constraints, we observe that minimizing $y_h^i$ gives us an optimal solution where $y_h^i = \mathrm{end_H}(c_i)$. We now determine the values of $y_h^i$ for all grid cells using the following system of vector inequalities:
\begin{align}
\boldsymbol{y_h} &\geq A_H\boldsymbol{x_h} \label{eq:a_h}\\
\boldsymbol{y_h} &\geq 0 \label{eq:nonneg}
\textrm{,}
\end{align}
where $\boldsymbol{y_h}$ and $\boldsymbol{x_h}$ are $n$-dimensional vectors composed of the variables $y_h^i$ and $x_h^i$ respectively. On further inspection, $A_H$ in (\ref{eq:a_h}) is the \emph{node-arc incidence (NAI) matrix} \cite{nemhauserIntegerCombinatorialOptimization1999} of the directed graph $G_H$ (see Fig.~\ref{fig:ILP_DiGraph}) composed of all horizontal path flows for an IOP grid cell from its left neighbour/border identifier. Similarly, we encode $\mathrm{end_V}(c)$ using the auxiliary variable vector $\boldsymbol{y_v}$ and the NAI matrix $A_V$ of the directed graph $G_V$ (Fig.~\ref{fig:ILP_DiGraph}). 

\begin{figure}%[t!]
\centering % <-- added
\includegraphics[width=0.4\textwidth]{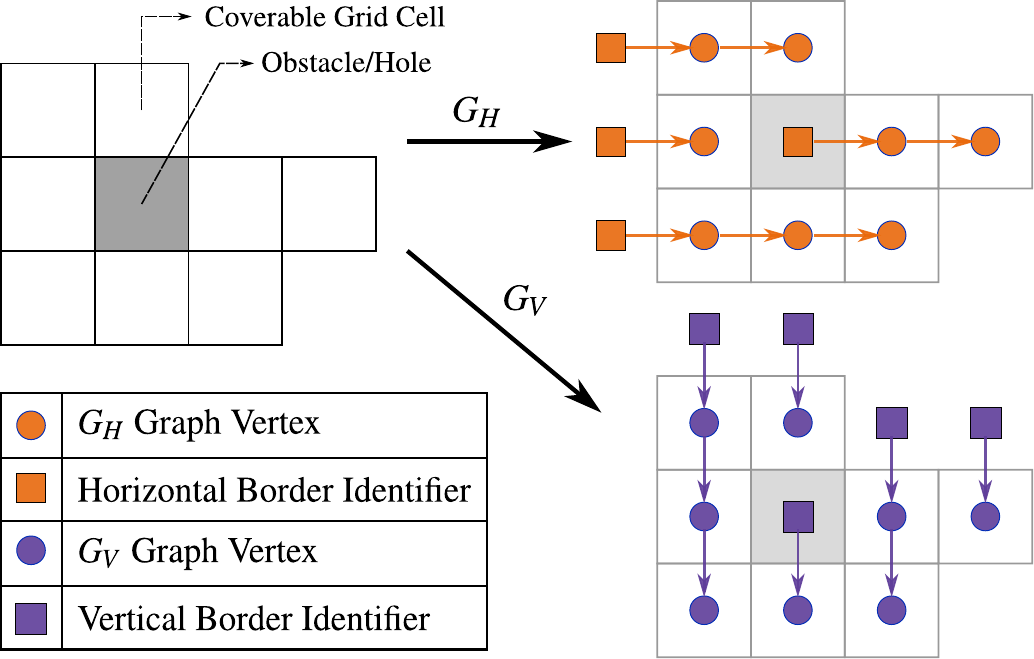}
\caption{The directed graphs $G_H$ and $G_V$ for an example IOP. 
\vspace{-0.5cm}}
\label{fig:ILP_DiGraph}
\end{figure}

Using the defined matrices and vectors, we write the following MILP to solve Problem \ref{problem:main}:
\begin{align}
\vspace{-10pt}
\min \sum_{i=0}^{n}y_h^i + \sum_{i=0}^{n}y_v^i \label{eq:eq2}\\
\textrm{s.t.} \quad 
A_H\boldsymbol{x_h} - \boldsymbol{y_h} \leq \boldsymbol{0} \label{eq:eq3}\\
A_V\boldsymbol{x_v} - \boldsymbol{y_v} \leq \boldsymbol{0} \label{eq:eq4}\\
\boldsymbol{x_h} + \boldsymbol{x_v} = \boldsymbol{1} \label{eq:eq5}\\
\boldsymbol{x_h}, \boldsymbol{x_v} \in \{0, 1\}^{n} & \hspace{4mm}
\boldsymbol{y_h}, \boldsymbol{y_v} \geq \boldsymbol{0} \label{eq:eq6}
\textrm{,}
\end{align}
where $\boldsymbol{1}$ is the column vector of ones and $\boldsymbol{0}$ is the vector of zeros. Eq. (\ref{eq:eq5}) ensures that each grid cell is assigned one of two orientations (either $x_h^i$ or $x_v^i$ equals 1). The objective in Eq. (\ref{eq:eq2}) gives us the number of ranks while $\boldsymbol{x_h}$ and $\boldsymbol{x_v}$ gives us the corresponding grid cell orientations. The solution of the MILP gives the minimum rank partition as determined by the grid cell orientations, where we obtain the ranks by merging all merge-able cells.

We also obtain an LP relaxation from this MILP using Eqs. (\ref{eq:eq2})-(\ref{eq:eq5}) and replacing the binary variables $\boldsymbol{x_h}$ and $\boldsymbol{x_v}$ in (\ref{eq:eq6}) with continuous variables as follows:
\begin{align}
% \min \sum_{i=0}^{n}\boldsymbol{y_h}[i] + \sum_{i=0}^{n}\boldsymbol{y_v}[i] \label{eq:eq14}\\
% \textrm{s.t.} \quad 
% A_H\boldsymbol{x_h} - \boldsymbol{y_h} \leq \boldsymbol{0} \label{eq:eq15}\\
% A_V\boldsymbol{x_v} - \boldsymbol{y_v} \leq \boldsymbol{0} \label{eq:eq16}\\
% \boldsymbol{x_h} + \boldsymbol{x_v} = \boldsymbol{1} \label{eq:eq17}\\
\boldsymbol{x_h}, \boldsymbol{x_v}, \boldsymbol{y_h}, \boldsymbol{y_v} \geq \boldsymbol{0} \label{eq:eq18}
\textrm{.}
\end{align}

\begin{proposition}
The relaxed LP denoted by Eqs. (\ref{eq:eq2})-(\ref{eq:eq5}),(\ref{eq:eq18}) computes an integral optimal solution and hence solves Problem \ref{problem:main} in polynomial time.
\label{proposition:main}
\end{proposition}

\section{Proof of Proposition \ref{proposition:main}} \label{section:poly}

In general, solving MILPs is NP-hard while LPs can be solved in polynomial time \cite{karmarkarNewPolynomialtimeAlgorithm1984}. An LP relaxation of a MILP in general will not yield optimal integral solutions. However, our LP relaxation belongs to a special class of problems that yield integral optimal solutions. This is mainly because of our use of NAI matrices, which are totally unimodular (TU) \cite{pitsoulisRepresentabilityTotallyUnimodular2009}. A matrix is TU if the determinants of all its square submatrices are in $\{-1, 0, 1\}$ \cite{nemhauserIntegerCombinatorialOptimization1999}. In this section, we prove that the matrix of all constraints in Eqs. (\ref{eq:eq2})-(\ref{eq:eq5}), (\ref{eq:eq18}) is TU. It will then follow from the Hoffman-Kruskal Principle \cite{hoffmanIntegralBoundaryPoints2010} that the relaxed LP can directly solve the MILP in polynomial time, since our matrix is TU.

\begin{theorem*}[Hoffman-Kruskal Principle \cite{hoffmanIntegralBoundaryPoints2010}]
\label{thm:Hoff-Krusk}
If an integral matrix A is TU, then for an integral vector $\boldsymbol{b}$, the polyhedron $\{\boldsymbol{x} \in \mathbb{R}^n \mid A\boldsymbol{x} \leq \boldsymbol{b}\}$ has integral coordinates.
\end{theorem*}

Since the polyhedron has integral coordinates, the resulting optimal solution for the LP defined in the polyhedron is also integral. 

\begin{lemma}
The constraints of the LP formulation in Eqs. (\ref{eq:eq2})-(\ref{eq:eq5}),(\ref{eq:eq18}) are TU.
\label{lemma:TUM}
\end{lemma}
\begin{proof}
First, we write our problem in Standard Equality Form (SEF) using slack variables to eliminate the inequality constraints. We then rewrite the summations in Eq. \ref{eq:eq2} using a vector multiplication with $\boldsymbol{1}$ and obtain
\begin{align}
\min \quad & \boldsymbol{1}^T\boldsymbol{y_h} + \boldsymbol{1}^T\boldsymbol{y_v} \label{eq:eq7}\\
\textrm{s.t.} 
\quad & A_H\boldsymbol{x_h} - \boldsymbol{y_h} + \boldsymbol{z_h} = \boldsymbol{0} \label{eq:eq8}\\
\quad & A_V\boldsymbol{x_v} - \boldsymbol{y_v} + \boldsymbol{z_v} = \boldsymbol{0} \label{eq:eq9} \\
\quad & \boldsymbol{x_h} + \boldsymbol{x_v}= \boldsymbol{1} \label{eq:eq10}\\
\quad & \boldsymbol{x_h}, \boldsymbol{x_v}, \boldsymbol{y_h}, \boldsymbol{y_v}, \boldsymbol{z_h}, \boldsymbol{z_v} \geq \boldsymbol{0} \label{eq:eq11}
\textrm{.}
\end{align}

The constraints are now of the form $A\boldsymbol{x} = \boldsymbol{b}$, where

\begin{equation}
\begin{aligned}
A = \begin{bmatrix}
A_H & \boldsymbol{\overline{0}} & -I & \boldsymbol{\overline{0}} & I & \boldsymbol{\overline{0}} \\
\boldsymbol{\overline{0}} & A_V & \boldsymbol{\overline{0}} & -I & \boldsymbol{\overline{0}} & I \\
I & I & \boldsymbol{\overline{0}} & \boldsymbol{\overline{0}} & \boldsymbol{\overline{0}} & \boldsymbol{\overline{0}}
\end{bmatrix}
\textrm{,}
\label{eq:eq15}
\end{aligned}
\end{equation}
\begin{align}
\quad & \boldsymbol{b} = \begin{bmatrix}
\boldsymbol{0}^T &
\boldsymbol{0}^T &
\boldsymbol{1}^T
\end{bmatrix}^T
\textrm{,}
\label{eq:eq25}
\end{align}
\begin{align}
\quad & \boldsymbol{x} = \begin{bmatrix}
\boldsymbol{x_h}^T &
\boldsymbol{x_v}^T &
\boldsymbol{y_h}^T &
\boldsymbol{y_v}^T &
\boldsymbol{z_h}^T &
\boldsymbol{z_v}^T
\end{bmatrix}^T
\textrm{,}
\end{align}

$\boldsymbol{\overline{0}}$ is the matrix of zeros and $I$ is the identity matrix. Clearly $\boldsymbol{b}$ is an integral vector. 
% So according to the Hoffman-Kruskal Principle, the LP relaxation yields integral solutions to $\boldsymbol{x}$ if $A$ is TU.

To prove $A$ is TU, let us start with the matrix
\begin{align*}
\Tilde{A} = \begin{bmatrix}
A_H^T & \boldsymbol{\overline{0}} & I \\
\boldsymbol{\overline{0}} & A_V^T & I
\end{bmatrix}
\textrm{.}
\end{align*}

Each row of the NAI matrices $A_H$ and $A_V$ signifies a directed edge in the graph, with a -1 for the source grid cell (outgoing), a +1 for the sink grid cell (incoming), and 0s otherwise \cite{nemhauserIntegerCombinatorialOptimization1999}.

Due to this construction, $\Tilde{A}$ satisfies the following sufficient conditions for TU matrices as outlined in \cite{hoffmanIntegralBoundaryPoints2010}.
\begin{enumerate}
    \item All elements are in \{-1, 0, +1\} (all entries in $A_V$ and $A_H$ are either -1, 0, or +1).
    \item Each column of $\tilde{A}$ has at most two non-zero elements (each column of $A_H^T$ and $A_V^T$ will have two non-zero entries, one for the source and one for the sink).
    \item There exists a partition of the rows of $\tilde{A}$ into two disjoint sets $T_1$ and $T_2$ such that:
        \begin{enumerate}[label=(\roman*)]
            \item If any column of $\tilde{A}$ contains two nonzero entries of the same sign, then one is in a row of $T_1$ and the other is in a row of $T_2$.
            \item If any column of $\tilde{A}$ contains two nonzero entries of the opposite sign, then they are both in a row of $T_1$ or in a row of $T_2$.
        \end{enumerate}
    For $\tilde{A}$, $T_1$ consists of all rows in
        \begin{align*}
        \begin{bmatrix}
        A_H^T &
        \boldsymbol{\overline{0}} & I
        \end{bmatrix}
        \end{align*} 
        and $T_2$ consists of all rows in
        \begin{align*}
        \begin{bmatrix}
        \boldsymbol{\overline{0}} &
        A_V^T &
        I
        \end{bmatrix}
        \textrm{.}
        \end{align*}
\end{enumerate}

It follows from \cite{nemhauserIntegerCombinatorialOptimization1999} that the transpose of $\Tilde{A}$ is also TU. We can construct $A$ from $\Tilde{A}^T$ by appending columns of the positive and negative signed unit matrices $\Bar{I}$ and $-\Bar{I}$ where
\[\Bar{I} = \begin{bmatrix}
I & \boldsymbol{\overline{0}} & \boldsymbol{\overline{0}} \\
\boldsymbol{\overline{0}} & I & \boldsymbol{\overline{0}} \\
\boldsymbol{\overline{0}} & \boldsymbol{\overline{0}} & I
\end{bmatrix}\textrm{.}\]
Appending these columns preserves the total unimodularity of the resulting matrix, see \cite{nemhauserIntegerCombinatorialOptimization1999}. Because of this, the matrix $A$ is TU, which completes the proof of Lemma \ref{lemma:TUM}.
\end{proof}

Since the constraints are TU, it follows from Theorem \ref{thm:Hoff-Krusk} that the relaxed LP computes optimal integral solutions for the formulated MILP, thereby completing the proof of Proposition \ref{proposition:main}. We can therefore use efficient polynomial-time LP solvers \cite{karmarkarNewPolynomialtimeAlgorithm1984} to solve Problem \ref{problem:main} optimally.

\section{Tour Generation}\label{section:gtsp}

We now address Problem \ref{problem:tour} and plan our full coverage path by computing a tour of the ranks obtained from the LP in Section \ref{section:milp}. We do this by formulating a GTSP to determine a visitation order for the ranks. Using a GTSP formulation offers us the flexibility to use existing approaches to compute the tour \cite{smithGLNSEffectiveLarge2017, noonEfficientTransformationGeneralized1993a}. We employ a similar formulation that was proposed in \cite{bochkarevMinimizingTurnsRobot2016} and create an auxiliary graph representing all possible connections between the ranks. The vertices of this graph are grouped into sets, where each set consists of two vertices representing the two directions to cover the rank (e.g., for a horizontal line we can cover it from left to right or right to left). The GTSP aims to minimize the cost of visiting one vertex in every set (traversing each rank in a specific direction).

The edge costs between the vertices in different sets of the auxiliary graph are given by the time to travel between rank endpoints along an obstacle-free \emph{transition path}.{\blue These transition paths are planned using the dynamics of the robot, which depends on the robot's design. For this work, we assume semi-holonomic dynamics where (i) the robot stops to turn in-place with a constant angular velocity, and (ii) the robot has piecewise constant acceleration (when accelerating or decelerating) while travelling in a straight line with a maximum linear velocity. The same dynamics model is used in \cite{vandermeulenTurnminimizingMultirobotCoverage2019} and \cite{bahnemannRevisitingBoustrophedonCoverage2021}, which we compare our method against in Section \ref{section:results}.} We compute each transition path by constructing a visibility graph within the IOP boundaries and planning the shortest path using an A* search \cite{obermeyerVisiLibityLibraryVisibility2008}.{\blue The resulting path is a sequence of straight lines and intermediary turns, for which the travel time is computed.}

% An extension of the tour generation step is to formulate it as a multiple GTSP (m-GTSP) and solve for a multi-robot coverage tour. 

%****************************************************************
% RESULTS
%****************************************************************
\section{Results}\label{section:results}

{\blue In this section, we present our experiments to test the performance of OARP and compare it to two different coverage planning approaches from literature. Finally, we present ROS simulations on an example environment with the Avidbots Neo robot model.}

% \textbf{We then run simulations using the Avidbots Neo robot to cover an example map.}}

{\blue Firstly, in Sections~\ref{decomp-comp} to \ref{tour-comp}, we compare OARP to the heuristic approach from \cite{vandermeulenTurnminimizingMultirobotCoverage2019}.} We chose the heuristic approach for our main comparisons because it produced the least turns and the lowest tour costs in our analysis of the available approaches.{\blue While the work in \cite{vandermeulenTurnminimizingMultirobotCoverage2019} has applied the heuristic approach for both single and multi-robot cases, the rank partitioning step is itself independent of the number of robots, which is similar to OARP. We focus on the single-robot coverage tours in this comparison to analyze how the different rank partitioning methods affect the tour.} For this comparison, we use a dataset containing 2D maps of 44 real-world environments obtained from Avidbots and used for their Neo cleaning robot. These environments have a minimum coverage path length ranging from 100 $m$ to around 3700 $m$, where the minimum coverage path length of an environment is given by $\textrm{(total area)}/(\textrm{tool width } l)$. The maps in this dataset are arranged/labeled by increasing complexity, i.e., the minimum number of ranks required to cover the environment. Due to the confidentiality of these scans, we cannot share this dataset but we have additionally generated a set of anonymized environments for visualization, which are shown in Figs. \ref{fig:rank-part} and \ref{fig:cov-tour}. Each method is run for 20 trials on each map in the dataset and we compare the results for every stage of the coverage planning process.

% {\color{orange} In Section \ref{bcd-comp}, we compare OARP to an alternative coverage planning method, specifically the boustrophedon cell decomposition (BCD) method from \cite{bahnemannRevisitingBoustrophedonCoverage2021}.} 
% Furthermore, we present robot simulations using an example map in Section \ref{robot-sim}.}

% \textbf{[Remove?]The environment from Fig.~\ref{fig:rank-part} is also used in our video submission to visualize coverage planning using OARP.}

\begin{figure}
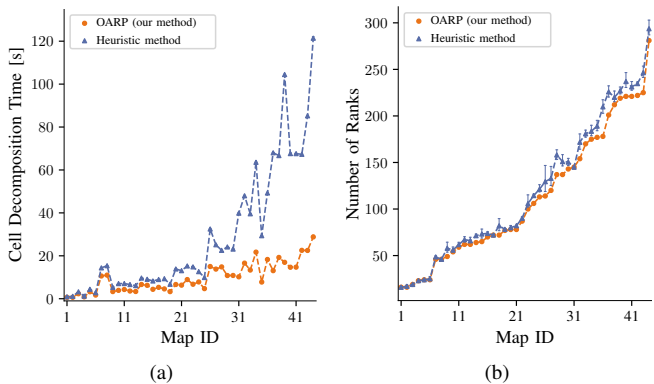
%[t!]
\centering
\subcaptionbox{}
{\centering
\resizebox{0.476\linewidth}{!}{\input{Assets/Results/av_preprocess.pgf}}
}
\hfill
\subcaptionbox{}
% {\includegraphics[width=0.49\linewidth]{Assets/Results/AV - Ranks.pdf}}
{\centering
\resizebox{0.476\linewidth}{!}{\input{Assets/Results/av_ranks.pgf}}}
\caption{Comparison of our rank partitioning method with that proposed in \cite{vandermeulenTurnminimizingMultirobotCoverage2019}: (a) Cell decomposition runtime for each map. (b) The number of ranks determined for each test environment.
\vspace{-0.5cm}}
\label{fig:rank_comp}
\end{figure}

\subsection{Decomposition Comparisons}
\label{decomp-comp}

For each map in our dataset, we compared the time it takes for both methods to decompose the environment into cells. The heuristic method decomposes the environment into a preliminary set of coarse rectangular cells \cite{vandermeulenTurnminimizingMultirobotCoverage2019}, while the OARP method simply needs to express the IOP of the environment using a grid overlay.{\blue To construct the IOP, we include all grid cells where over 50\% of its area is contained within the original environment.} The results of the comparison are shown in Fig.~\ref{fig:rank_comp}a, where we observe that as the complexity of the map increases, the cell decomposition time for the heuristic method increases rapidly. In contrast, OARP only requires a grid approximation which can be computed much faster; about 50\% faster on average and 5 times faster in some cases.{\blue This makes a large impact on the overall planner time, especially in cases where the coverage plan would need to be replanned due to perceived changes or uncertainty in the environment. Robustification of our approach to dynamic environment uncertainties is a future research direction.} 

\subsection{Rank Partitioning Comparisons}
\label{rank-comp}

\begin{figure}[t!]
\centering
\includegraphics[width=0.68\linewidth]{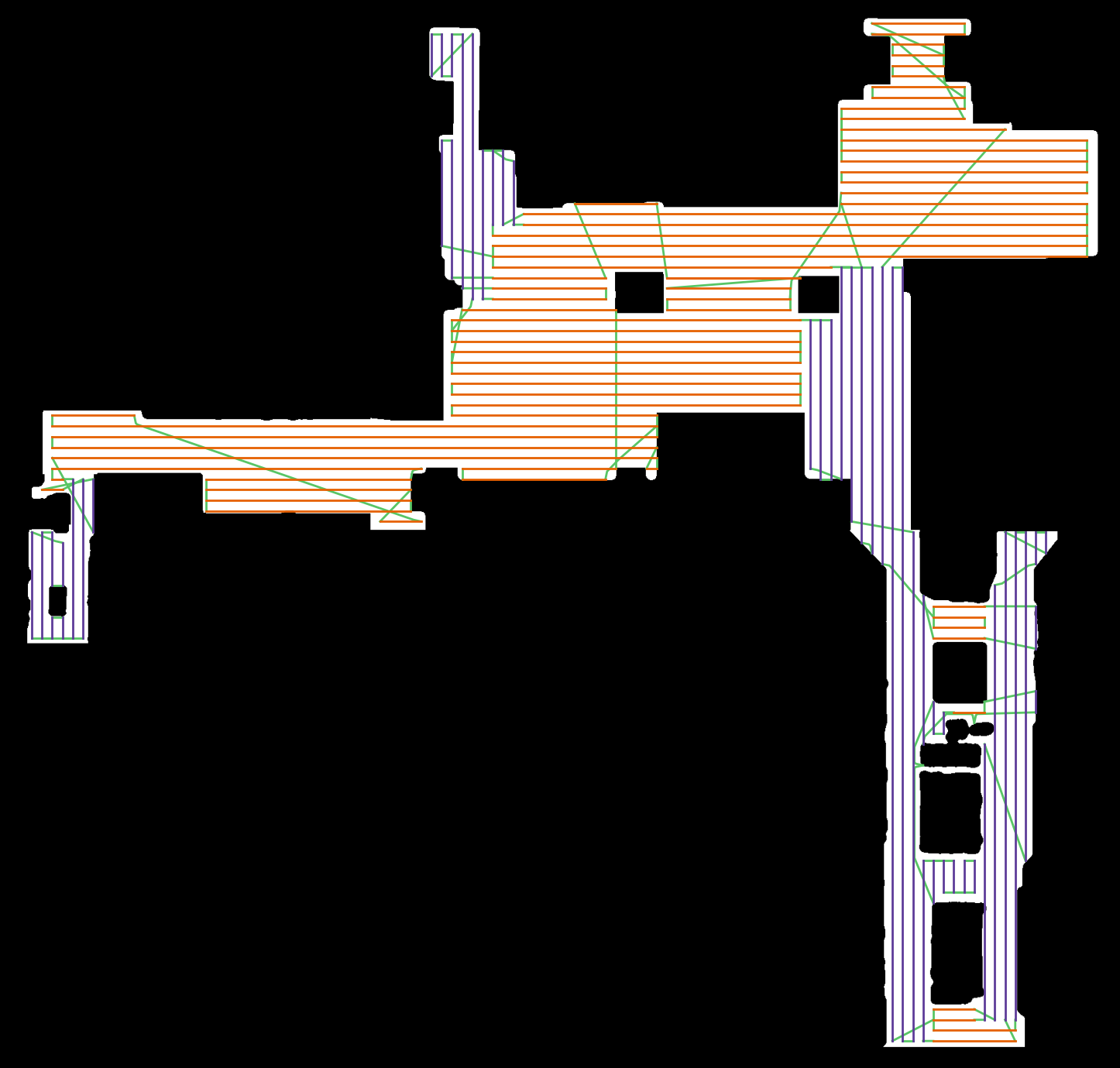}
\caption{Coverage path of an indoor environment generated using the proposed LP method.
\vspace{-0.5cm}}
\label{fig:cov-tour}
\end{figure}

We implement our rank partitioning step by expressing the LP using Eqs. (\ref{eq:eq2})-(\ref{eq:eq5}), Eq. (\ref{eq:eq18}), and the IOP built in the decomposition step. The algorithm was prototyped in Python and we used the free LP solver, GLPK \cite{makhorinGLPKGNUProject} to solve the LP. The coverage tool width was set to 0.8 $m$ to model Avidbots' robot. Fig.~\ref{fig:rank-part} shows the rank partitioning of an example map with the horizontal (orange) and vertical (purple) ranks as determined by the LP solver. Solving the LP takes an average of 0.4 seconds across all trials for all maps.

For the heuristic method, we solve the same IOP coverage problem by only including the interior ranks from their method (we ignore the perimeter ranks). The heuristic method is an iterative algorithm, where running for more iterations improves the performance. So, one would need to allow enough iterations to find high quality solutions, but avoid extra iterations after the optimal has been found. In contrast, OARP avoids unnecessary computation by deterministically finding an optimal solution before terminating.

%But, the solver does not know when it has reached the optimal solution.

To standardize the comparison, we limit the runtime of the heuristic method's partitioning step to that of OARP's partitioning step (the LP's runtime) for each map. Fig.~\ref{fig:rank_comp}b shows the results of the experiments. As expected, we observe that the number of ranks obtained by OARP is always as good or better than the heuristic method, with the performance gap widening for more complex maps. In general, the heuristic method does quite well, likely due to the tractability of the problem. However, since this problem is tractable and thus solvable in polynomial time, one should choose an optimal solver.

\subsection{Coverage Tour Comparisons}
\label{tour-comp}

\begin{table}%[t!]
\vspace{6pt}
\centering
\caption{Robot parameters used for experiments}
\label{table:parameters}
\begin{tabular}{ l | l}
\toprule
Coverage Tool Width & 0.8 $m$\\
% \hline
Maximum Linear Velocity & 1 $m/s$\\
% \hline
Linear Acceleration & $\pm 0.5$ $m/s^2$ \\
% \hline
Angular Velocity &$30^{\circ}/s$ \\
\bottomrule
\end{tabular}
\end{table}

To generate the coverage tour from the ranks, we use the GTSP formulation in Section \ref{section:gtsp} to compute a tour that minimizes the transition time between the ranks. Table \ref{table:parameters} documents our choice of parameters used to compute the transition costs by simulating the robot's coverage tool and dynamics. We used GLNS \cite{smithGLNSEffectiveLarge2017} to solve the formulated GTSP problem for both OARP and the heuristic methods. Fig.~\ref{fig:cov-tour} shows an example coverage tour for an anonymized environment from our dataset, with the transitions (green lines) between the ranks as determined by the GTSP solver.

We also compared our tours to those obtained using the heuristic method from \cite{vandermeulenTurnminimizingMultirobotCoverage2019}. Since the ranks are non-overlapping for both methods, the cost of the tour is related to the number of turns. Fig.~\ref{fig:improvement} shows the comparison results, where we plot the average percent improvement of OARP over the heuristic method for each map on two metrics: a) number of turns, and b) total coverage tour cost. We observe that OARP improves the coverage tours on both metrics for almost all maps. For a small number of maps, we see little to no improvement due to the heuristic method providing the same number of ranks as OARP for those maps.{\blue On average for all maps, OARP improved the number of turns by 5.9\% and the tour cost by around 2.7\%. For some maps, we get a 16.8\% improvement in the number of turns, and a 7.6\% improvement in total coverage time.} From this comparison, we observe the intended trend that less ranks leads to less turns, and less turns leads to shorter coverage paths.

\subsection{Comparison against BCD}
\label{bcd-comp}

\begin{figure}
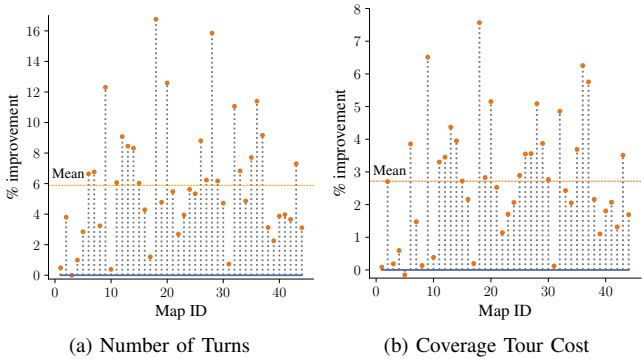
%[t!]
\centering
\subcaptionbox{Number of Turns}
{\centering
\resizebox{0.475\linewidth}{!}{\input{Assets/Results/av_turn_comp.pgf}}}
\hfill
\subcaptionbox{Coverage Tour Cost}
{\centering
\resizebox{0.475\linewidth}{!}{\input{Assets/Results/av_tour_cost_comp.pgf}}
}
\hfill
\caption{\blue Average \% improvement exhibited by the OARP method in comparison with \cite{vandermeulenTurnminimizingMultirobotCoverage2019} for each map.
\vspace{-0.5cm}}
\label{fig:improvement}
\end{figure}

\begin{figure}%[t!]
\centering

\resizebox{0.7\linewidth}{!}{\input{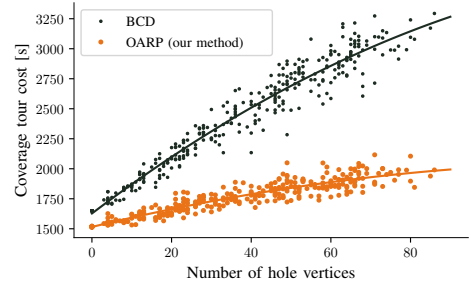}}
\caption{\blue Comparison results of the coverage tour generated using OARP and the BCD method \cite{bahnemannRevisitingBoustrophedonCoverage2021}. The x-axis specifies the number of vertices representing the holes in the maps.
\vspace{-0.1cm}}
\label{fig:bcd-comp}
\end{figure}

{We additionally compared the proposed OARP method against the Boustrophedon Cell Decomposition (BCD) method from \cite{bahnemannRevisitingBoustrophedonCoverage2021}.} For this comparison, we used the dataset from \cite{bahnemannRevisitingBoustrophedonCoverage2021} containing over 300 aerial scans of environments with varying number of holes (obstacles). The complexity measure of the maps in this dataset is given by \textit{hole vertices}, i.e., the total number of vertices required to represent the holes in the map (e.g., one triangular hole has 3 vertices). The maximum number of hole vertices in this dataset is 86, whereas in Avidbots' dataset, this number goes up to 900. We do not test BCD on Avidbots' dataset as it would require {significant} modifications to the BCD implementation, but we expect to see a similar but exaggerated trend due to the large number of hole vertices.{\blue To perform this comparison, we use the same parameters from \cite{bahnemannRevisitingBoustrophedonCoverage2021}; a tool width of 3 $m$, linear acceleration of $\pm 1$ $m/s^2$, and maximum velocity of 3 $m/s$. We also penalize turns by adding turning time to the cost function, where the robot turns in-place with an angular velocity of $30^{\circ}/s$.} Fig.~\ref{fig:bcd-comp} shows the results of comparing the cost of the coverage tours generated by OARP and BCD using identical parameters.{\blue We observed that OARP outperforms BCD for all maps, and improves the coverage tour by 25.5\% on average.} We believe that this performance gap is due to OARP's better coverage line placement. We also tested the heuristic method from \cite{vandermeulenTurnminimizingMultirobotCoverage2019} on this dataset and observed that, while OARP still outperforms~\cite{vandermeulenTurnminimizingMultirobotCoverage2019}, the average performance improvement is small ($<$1\%) due to the relative simplicity of the maps in this dataset.

{\blue
\subsection{ROS Simulation Case Study}
\label{robot-sim}

\begin{figure}
\centering
\subcaptionbox{}
{\centering
\includegraphics[width=0.42\linewidth]{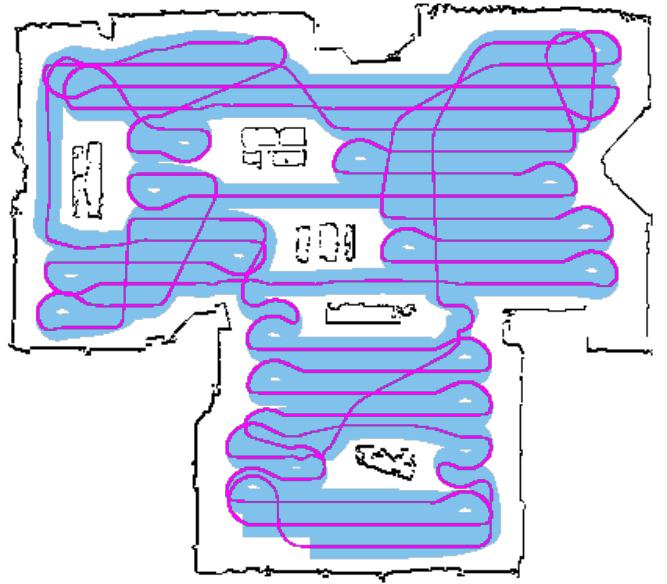}}
\hfill
\subcaptionbox{}
{\centering
\includegraphics[width=0.42\linewidth]{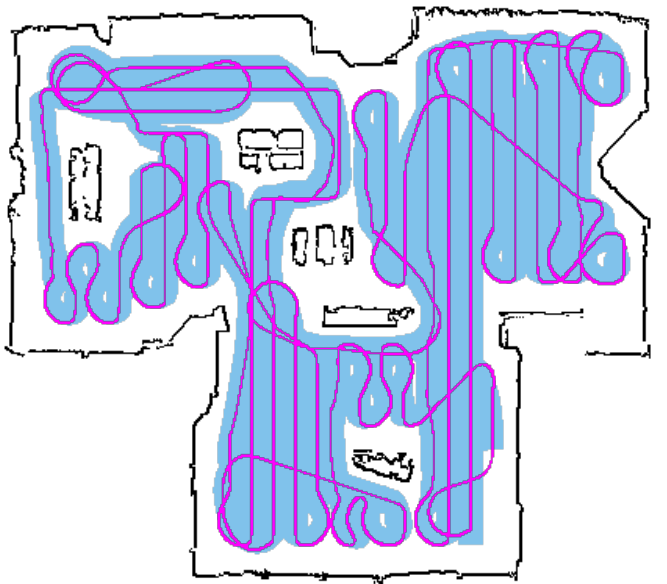}}
\hfill
\caption{\blue Simulation results of covering an example environment using the Avidbots Neo robot. The covarage planners used were (a) OARP and (b) the heuristic method from \cite{vandermeulenTurnminimizingMultirobotCoverage2019}.
\vspace{-0.5cm}}
\label{fig:robot-sim-res}
\end{figure}

In this subsection, we provide a case study of simulating the coverage of an example environment using the Avidbots Neo robot. The simulation was conducted using ROS on an anonymized hybrid environment as shown in Fig.~\ref{fig:robot-sim-res}. The Neo robot has tricycle dynamics and so coverage was planned with sufficient space at the boundaries to allow comfortable turns. The robot was also equipped with a local planner to replan difficult parts of the coverage plan (e.g. areas with small coverage lines) if necessary. Fig.~\ref{fig:robot-sim-res} shows the resulting environment coverage using plans generated by OARP and the heuristic method \cite{vandermeulenTurnminimizingMultirobotCoverage2019}. We observed that OARP's plan was nearly 13\% faster than the heuristic method's plan and 3.3\% shorter in path length. We also observed that OARP's plan covered more area of this map than the heuristic method's plan (1.1\% more). The reasoning for this is two-fold: OARP's path consists of less turns (poor coverage at turns), and OARP's plan required less replans by the local planner. We also attempted to test the plan generated using the BCD planner from \cite{bahnemannRevisitingBoustrophedonCoverage2021}, but the plan contained many sharp turns that are infeasible for the Neo robot. Recordings of the simulations are also included in our video attachment.

% The real-life Neo robot has tricycle dynamics and struggles to clean in , so we first evaluated the plans using a simulation with relaxed robot dynamics,
% We observe that while both plans take approximately the same time to execute (20mins), the OARP plan required less replans and therefore covered more area of this map (41$\textrm{ft}^2$ difference) than the heuristic plan. 

}

\section{Conclusion}

In this paper, we proposed the OARP coverage planning method that aims to minimize the number of turns needed to cover the robot's environment. We do this by partitioning the environment into thin axis-parallel ranks using a linear program (LP). We proved that this formulation computes the optimal rank partitioning in polynomial time. We then conducted experiments on maps of real-world environments and showed that OARP always achieves the best rank partition in comparison with the state-of-the-art method. Furthermore, OARP is also shown to have around 6\% lesser turns and 3\% shorter coverage tours on average than the other method. 

{\blue While OARP can be used for any environment, applying the axis-parallel constraint may result in "stair-case" paths to cover narrow areas with non-axis-parallel or curved boundaries. A future research direction is relaxing this constraint to improve coverage of such areas.} We also aim to leverage the performance of OARP to develop an online turn-minimizing coverage planner for uncertain dynamic environments.

% % use section* for acknowledgement
\section*{Acknowledgment}
This work was supported by Canadian Mitacs Accelerate Project IT16435 and Avidbots Corp, Kitchener, ON, Canada.

% \newpage{}

\bibliographystyle{IEEEtran}
\bibliography{IEEEabrv,references}

\begin{thebibliography}{10}
\providecommand{\url}[1]{#1}
\csname url@samestyle\endcsname
\providecommand{\newblock}{\relax}
\providecommand{\bibinfo}[2]{#2}
\providecommand{\BIBentrySTDinterwordspacing}{\spaceskip=0pt\relax}
\providecommand{\BIBentryALTinterwordstretchfactor}{4}
\providecommand{\BIBentryALTinterwordspacing}{\spaceskip=\fontdimen2\font plus
\BIBentryALTinterwordstretchfactor\fontdimen3\font minus
  \fontdimen4\font\relax}
\providecommand{\BIBforeignlanguage}[2]{{%
\expandafter\ifx\csname l@#1\endcsname\relax
\typeout{** WARNING: IEEEtran.bst: No hyphenation pattern has been}%
\typeout{** loaded for the language `#1'. Using the pattern for}%
\typeout{** the default language instead.}%
\else
\language=\csname l@#1\endcsname
\fi
#2}}
\providecommand{\BIBdecl}{\relax}
\BIBdecl

\bibitem{galceranSurveyCoveragePath2013}
E.~Galceran and M.~Carreras, ``A survey on coverage path planning for
  robotics,'' \emph{Robotics and Autonomous Systems}, vol.~61, no.~12, pp.
  1258--1276, 2013.

\bibitem{hofnerPathPlanningGuidance1995}
C.~Hofner and G.~Schmidt, ``Path planning and guidance techniques for an
  autonomous mobile cleaning robot,'' \emph{Robotics and Autonomous Systems},
  vol.~14, no.~2, pp. 199--212, 1995.

\bibitem{hameedIntelligentCoveragePath2014}
I.~A. Hameed, ``Intelligent coverage path planning for agricultural robots and
  autonomous machines on three-dimensional terrain,'' \emph{Journal of
  Intelligent and Robotic Systems: Theory and Applications}, vol.~74, no. 3-4,
  pp. 965--983, 2014.

\bibitem{songOnlineCoverageInspection2020}
S.~Song, D.~Kim, and S.~Jo, ``Online coverage and inspection planning for
  {{3D}} modeling,'' \emph{Autonomous Robots}, vol.~44, no.~8, pp. 1431--1450,
  2020.

\bibitem{jingCoveragePathPlanning2019}
W.~Jing, D.~Deng, Z.~Xiao, Y.~Liu, and K.~Shimada, ``Coverage {{Path Planning}}
  using {{Path Primitive Sampling}} and {{Primitive Coverage Graph}} for
  {{Visual Inspection}},'' in \emph{{{IEEE}}/{{RSJ International Conference}}
  on {{Intelligent Robots}} and {{Systems}} ({{IROS}})}, 2019, pp. 1472--1479.

\bibitem{biundiniFrameworkCoveragePath2021}
I.~Z. Biundini, M.~F. Pinto, A.~G. Melo, A.~L.~M. Marcato, L.~M. Hon{\'o}rio,
  and M.~J.~R. Aguiar, ``A {{Framework}} for {{Coverage Path Planning
  Optimization Based}} on {{Point Cloud}} for {{Structural Inspection}},''
  \emph{Sensors}, vol.~21, no.~2, p. 570, 2021.

\bibitem{nasirianEfficientCoveragePath2021}
B.~Nasirian, M.~Mehrandezh, and F.~{Janabi-Sharifi}, ``Efficient {{Coverage
  Path Planning}} for {{Mobile Disinfecting Robots Using Graph-Based
  Representation}} of {{Environment}},'' \emph{Frontiers in Robotics and AI},
  vol.~8, p.~4, 2021.

\bibitem{dasMappingPlanningSample2014}
A.~Das, M.~Diu, N.~Mathew, C.~Scharfenberger, J.~Servos, A.~Wong, J.~S. Zelek,
  D.~A. Clausi, and S.~L. Waslander, ``Mapping, {{Planning}}, and {{Sample
  Detection Strategies}} for {{Autonomous Exploration}},'' \emph{Journal of
  Field Robotics}, vol.~31, no.~1, pp. 75--106, 2014.

\bibitem{arkinOptimalCoveringTours2005}
E.~M. Arkin, M.~A. Bender, E.~D. Demaine, S.~P. Fekete, J.~S.~B. Mitchell, and
  S.~Sethia, ``Optimal {{Covering Tours}} with {{Turn Costs}},'' \emph{SIAM
  Journal on Computing}, vol.~35, no.~3, pp. 531--566, 2005.

\bibitem{chosetCoveragePathPlanning1998}
H.~Choset and P.~Pignon, ``Coverage {{Path Planning}}: {{The Boustrophedon
  Cellular Decomposition}},'' \emph{Field and Service Robotics}, pp. 203--209,
  1998.

\bibitem{vandermeulenTurnminimizingMultirobotCoverage2019}
I.~Vandermeulen, R.~Gro{\ss}, and A.~Kolling, ``Turn-minimizing multirobot
  coverage,'' in \emph{2019 {{IEEE International Conference}} on {{Robotics}}
  and {{Automation}} ({{ICRA}})}, 2019, pp. 1014--1020.

\bibitem{bormannIndoorCoveragePath2018}
R.~Bormann, F.~Jordan, J.~Hampp, and M.~H{\"a}gele, ``Indoor {{Coverage Path
  Planning}}: {{Survey}}, {{Implementation}}, {{Analysis}},'' in \emph{2018
  {{IEEE International Conference}} on {{Robotics}} and {{Automation}}
  ({{ICRA}})}, May 2018, pp. 1718--1725.

\bibitem{bochkarevMinimizingTurnsRobot2016}
S.~Bochkarev and S.~L. Smith, ``On minimizing turns in robot coverage path
  planning,'' in \emph{2016 {{IEEE International Conference}} on {{Automation
  Science}} and {{Engineering}} ({{CASE}})}, 2016, pp. 1237--1242.

\bibitem{bahnemannRevisitingBoustrophedonCoverage2021}
R.~B{\"a}hnemann, N.~Lawrance, J.~J. Chung, M.~Pantic, R.~Siegwart, and
  J.~Nieto, ``Revisiting {{Boustrophedon Coverage Path Planning}} as a
  {{Generalized Traveling Salesman Problem}},'' in \emph{Field and {{Service
  Robotics}}}, ser. Springer {{Proceedings}} in {{Advanced Robotics}}.\hskip
  1em plus 0.5em minus 0.4em\relax {Singapore}: {Springer}, 2021, pp. 277--290.

\bibitem{torresCoveragePathPlanning2016}
M.~Torres, D.~A. Pelta, J.~L. Verdegay, and J.~C. Torres, ``Coverage path
  planning with unmanned aerial vehicles for {{3D}} terrain reconstruction,''
  \emph{Expert Systems with Applications}, vol.~55, pp. 441--451, 2016.

\bibitem{j.jinOptimalCoveragePath2010}
{J. Jin} and {L. Tang}, ``Optimal {{Coverage Path Planning}} for {{Arable
  Farming}} on {{2D Surfaces}},'' \emph{Transactions of the ASABE}, vol.~53,
  no.~1, pp. 283--295, 2010.

\bibitem{cabreiraGridBasedCoveragePath2019}
T.~M. Cabreira, P.~R. Ferreira, C.~D. Franco, and G.~C. Buttazzo,
  ``Grid-{{Based Coverage Path Planning With Minimum Energy Over
  Irregular-Shaped Areas With UAVs}},'' in \emph{2019 {{International
  Conference}} on {{Unmanned Aircraft Systems}} ({{ICUAS}})}, 2019, pp.
  758--767.

\bibitem{kanOnlineExplorationCoverage2020}
X.~Kan, H.~Teng, and K.~Karydis, ``Online {{Exploration}} and {{Coverage
  Planning}} in {{Unknown Obstacle-Cluttered Environments}},'' \emph{IEEE
  Robotics and Automation Letters}, vol.~5, no.~4, pp. 5969--5976, 2020.

\bibitem{gabrielySpanningtreeBasedCoverage2001}
Y.~Gabriely and E.~Rimon, ``Spanning-tree based coverage of continuous areas by
  a mobile robot,'' \emph{Annals of Mathematics and Artificial Intelligence},
  vol.~31, no. 1-4, pp. 77--98, 2001.

\bibitem{theileUAVCoveragePath2020a}
M.~Theile, H.~Bayerlein, R.~Nai, D.~Gesbert, and M.~Caccamo, ``{{UAV Coverage
  Path Planning}} under {{Varying Power Constraints}} using {{Deep
  Reinforcement Learning}},'' in \emph{{{IEEE}}/{{RSJ International
  Conference}} on {{Intelligent Robots}} and {{Systems}} ({{IROS}})}, 2020, pp.
  1444--1449.

\bibitem{kyawCoveragePathPlanning2020}
P.~T. Kyaw, A.~Paing, T.~T. Thu, R.~E. Mohan, A.~Vu~Le, and
  P.~Veerajagadheswar, ``Coverage {{Path Planning}} for {{Decomposition
  Reconfigurable Grid-Maps Using Deep Reinforcement Learning Based Travelling
  Salesman Problem}},'' \emph{IEEE Access}, vol.~8, pp. 225\,945--225\,956,
  2020.

\bibitem{apuroopReinforcementLearningBasedComplete2021a}
K.~G.~S. Apuroop, A.~V. Le, M.~R. Elara, and B.~J. Sheu, ``Reinforcement
  {{Learning-Based Complete Area Coverage Path Planning}} for a {{Modified
  hTrihex Robot}},'' \emph{Sensors}, vol.~21, no.~4, p. 1067, 2021.

\bibitem{krishnalakshmananCompleteCoveragePath2020}
A.~Krishna~Lakshmanan, R.~Elara~Mohan, B.~Ramalingam, A.~Vu~Le,
  P.~Veerajagadeshwar, K.~Tiwari, and M.~Ilyas, ``Complete coverage path
  planning using reinforcement learning for {{Tetromino}} based cleaning and
  maintenance robot,'' \emph{Automation in Construction}, vol. 112, p. 103078,
  2020.

\bibitem{nemhauserIntegerCombinatorialOptimization1999}
G.~L. Nemhauser, \emph{Integer and Combinatorial Optimization}, ser.
  Wiley-{{Interscience}} Series in Discrete Mathematics and Optimization.\hskip
  1em plus 0.5em minus 0.4em\relax {New York}: {Wiley}, 1999.

\bibitem{karmarkarNewPolynomialtimeAlgorithm1984}
N.~Karmarkar, ``A new polynomial-time algorithm for linear programming,''
  \emph{Combinatorica}, vol.~4, no.~4, pp. 373--395, 1984.

\bibitem{pitsoulisRepresentabilityTotallyUnimodular2009}
L.~Pitsoulis, K.~Papalamprou, G.~Appa, and B.~Kotnyek, ``On the
  representability of totally unimodular matrices on bidirected graphs,''
  \emph{Discrete Mathematics}, vol. 309, no.~16, pp. 5024--5042, 2009.

\bibitem{hoffmanIntegralBoundaryPoints2010}
A.~J. Hoffman and J.~B. Kruskal, ``Integral {{Boundary Points}} of {{Convex
  Polyhedra}},'' in \emph{50 {{Years}} of {{Integer Programming}} 1958-2008:
  {{From}} the {{Early Years}} to the {{State-of-the-Art}}}.\hskip 1em plus
  0.5em minus 0.4em\relax {Berlin, Heidelberg}: {Springer}, 2010, pp. 49--76.

\bibitem{smithGLNSEffectiveLarge2017}
S.~L. Smith and F.~Imeson, ``{{GLNS}}: {{An}} effective large neighborhood
  search heuristic for the {{Generalized Traveling Salesman Problem}},''
  \emph{Computers \& Operations Research}, vol.~87, pp. 1--19, 2017.

\bibitem{noonEfficientTransformationGeneralized1993a}
C.~E. Noon and J.~C. Bean, ``An {{Efficient Transformation Of The Generalized
  Traveling Salesman Problem}},'' \emph{INFOR: Information Systems and
  Operational Research}, vol.~31, no.~1, pp. 39--44, 1993.

\bibitem{obermeyerVisiLibityLibraryVisibility2008}
K.~J. Obermeyer and {Contributors}, ``{{VisiLibity}}: {{A C}}++ {{Library}} for
  {{Visibility Computations}} in {{Planar Polygonal Environments}},'' 2008.

\bibitem{makhorinGLPKGNUProject}
A.~O. Makhorin, ``{{GLPK}} - {{GNU Project}} - {{Free Software Foundation}}
  ({{FSF}}),'' https://www.gnu.org/software/glpk/.

\end{thebibliography}

\end{document}